\documentclass{article}
\usepackage{times}

\usepackage{soul}
\usepackage{url}
\usepackage[hidelinks]{hyperref}
\usepackage[utf8]{inputenc}
\usepackage[small]{caption}
\usepackage{graphicx}
\usepackage{amsmath}
\usepackage{booktabs}
\usepackage[utf8]{inputenc} 
\usepackage[T1]{fontenc}    
\usepackage{hyperref}       
\usepackage{url}            
\usepackage{booktabs}       
\usepackage{amsfonts}       
\usepackage{nicefrac}       
\usepackage{microtype}      

\usepackage{graphicx,color,amsmath,amsthm,comment}
\usepackage[boxed]{algorithm2e}
\newtheorem{definition}{Definition}[section]
\newtheorem{remark}{Remark}[section]

\newtheorem{theorem}{Theorem}[section]
\newtheorem{lemma}{Lemma}[section]

\newcommand{\todo}[1]{\textcolor{red}{[#1]}}
 
\DeclareMathOperator*{\argmax}{argmax}
\DeclareMathOperator{\R}{\mathbb{R}}
\newcommand{\narms}{N}
\newcommand{\arms}{\mathcal{N}}

\newcommand{\pfm}{V}
\newcommand{\regret}{R_{T}}
\newcommand{\est}{\hat{\mu}}
\newcommand{\PayToAddNoise}{\textsc{CBwHeterogeniety}}
\newcommand{\PaymentProfile}{\textsc{CBwPayments}}

\newcommand{\LimitedBudget}{\textsc{CBChainedRestricted}}
\newcommand{\UnrestrictedBudget}{\textsc{CBChainedUnrestricted}}

\title{Incentivising Exploration and Recommendations for Contextual Bandits with Payments}

\author{
Priyank Agrawal\\ 
University of Illinois Urbana-Champaign\and
Theja Tulabandhula\\
University of Illinois Chicago\\
}

\begin{document}

\maketitle

\begin{abstract}
We propose a contextual bandit based model to capture the learning and social welfare goals of a web platform in the presence of myopic users. By using payments to incentivize these agents to explore different items/recommendations, we show how the platform can learn the inherent attributes of items and achieve a sublinear regret while maximizing cumulative social welfare. We also calculate theoretical bounds on the cumulative costs of incentivization to the platform. Unlike previous works in this domain, we consider contexts to be completely adversarial, and the behavior of the adversary is unknown to the platform. Our approach can improve various engagement metrics of users on e-commerce stores, recommendation engines and matching platforms.
\end{abstract}

\section{Introduction}\label{sec:introduction}

In several practical applications such as recommendation systems (mobile health apps, Netflix, Amazon product recommendations) and matching platforms (Uber, Taskrabbit, Upwork, Airbnb), the platform/firm has to learn various system parameters to optimize resource allocation while only \emph{partially} being able to control learning rates. This is because, the users who transact on such platforms can take autonomous actions that maximize their own utility based on potentially inaccurate information, sometimes to the detriment of the learning goals. 

It is well known that users are influenced by the ratings and reviews of previous users provided by the platform while making their purchase decisions on e-commerce platforms. While the platform can reveal such attributes of different items it sells, a myopic user's decision based on these attributes can be sub-optimal if attributes have not been learned well enough from previous transactions. Because of the positive feedback loop, the platform's estimates of these attributes may be very different from their true values, leading to loss of social welfare. While users are myopic, the platform tends to be long-term focused, and has to incentivise its users through discounts, promotions and other controls to learn these attributes accurately and increase the overall social welfare.

Similarly in the area of mobile health apps (e.g., for chronic care management, fitness \& general health, medication management) incentivization in learning can help the app serve users better, but might get impeded by users being immediate reward focused. Here, the platform typically sends recommendations for users to partake in activities with the goal of improved health outcomes~\cite{dantzig2013toward}. The quality of recommendations can be high if the platform knows the utility model of the users and their preferences for different activities. To learn these preferences, the platform could devise incentives to nudge the user to prefer a different activity than their currently preferred choice, where the latter is based on current low quality recommendations. If it can restrict the amount of nudging while still being able to learn enough to give good activity recommendations (based on what it has learned so far), then all users will be better off.

In the above two applications and many others, the platform's goal is to maximize social welfare of the myopic users by learning the system parameters just enough to make the best recommendations (or equivalently, ensuring that the users take the best actions for their contexts) over time, when compared to the clairvoyant benchmark of making recommendations when the system parameters are known. The paper focuses on modeling a \emph{principal-agent} variation of online learning  in the \emph{contextual bandit} setting that allows the platform (principal) to use payments as auxiliary controls. Typically, the platform needs to give payments (which are costly) since in most practical settings the choices of the users based on current data may not be exploratory enough. Our objective then is to design such payments schemes that allow learning and improving social welfare, while simultaneously not costing too much to the platform.

Contextual bandits, a popular framework to learn and maximize revenue in online advertising and recommendation domains~\cite{lattimore2018bandit,bietti2018contextual,riquelme2018deep}, are problems where users are modeled as contexts (feature vectors) and the learner picks an action tailored to the context for which it is rewarded (bandit feedback). The methods developed here learn the parameters of the reward generation model while simultaneously exploiting current information on the quality of the arms (popular algorithms include EXP4, $\epsilon$-greedy, RegCB etc). While limited in their expressivity compared to Markov Decision Processes (MDPs)  (there are no states), they tend to capture learning problems where the reward for an action (such as purchasing an item or walking for 10 minutes or standing up) has an immediate outcome (such as a positive utility or a better mood) fairly accurately. While MDPs are also a suitable approach, they are typically harder to learn and analyse theoretically. 

Only a few works have considered the principal-agent variations which involves incentivization in learning through payments or otherwise in the recent past. In ~\cite{chen2018incentivizing} show that a constant amount of payments is enough if the users are heterogeneous, however, in their setting the platform is aware of the arriving contexts and the distribution from which contexts are drawn. The role of user heterogeneity is further explored in~\cite{bastani2017mostly} and ~\cite{kannan18} as \emph{covariate diversity}. In the former work, the authors consider contexts to be stochastic and prove that myopic arm selection is enough for certain distributions of contexts when the number of arms is two, while in the later, the authors use controlled and known perturbations to the contexts and show that greedy (myopic) selection of arms gives sub-linear regret. In~\cite{kannan2017fairness}, the authors propose a randomized algorithm without an explicit user heterogeneity criteria. However, their technique requires use of ridge estimator to estimate arm attributes leading to unbiassed estimates.

A related but orthogonal approach is pursued in~\cite{mansour2015bayesian,mansour2016bayesian,cohen2018optimal,immorlica2019bayesian}, where the authors consider principal-agent settings but only allow the use of information asymmetry under incentive compatibility constraints to explore, unlike payments in our setting. A similar setting was also investigated in~\cite{immorlica2018incentivizing} where they explore various unbiased disclosure policies that the platform can use to explore. In ~\cite{han2015incentivizing} the authors also consider a principal-agent setting, and assuming that the principal knows the distribution from which the contexts arrive as well as that each arm is preferred by at least some contexts, provide regret and payment bounds for an incentivization algorithm (building on their earlier results in ~\cite{frazier2014incentivizing}). In a vanilla multi-armed bandit setting, the authors in ~\cite{wang2018multi} have studied how payments can help explore and achieve sublinear regret. 

\textit{Main contributions:} First, we propose a contextual bandit based principal-agent model where payments can be used as auxiliary controls to induce exploration and learning. Second, we develop qualitative and quantitative characterization of payments as means of ensuring exploratory behaviour by agents. We develop a novel algorithm and show that the expected aggregate payments it makes in such regimes is sub-linear in the time horizon $T$. Finally, we compare regret performance and payments requirements of our approach and other competitors on both synthetic and real datasets. We find that the greedy approach with no payments (i.e., the platform does not explore at all) work well with real data, however, there are synthetic data instances where its regret performance is consistently surpassed by algorithms such as ours.  Our proposed algorithm works with the most general agent behavior (adversarial contexts), moreover, the payments scheme does not require the principal to have the knowledge of the current context (see section \ref{sec:problem_statement}).

\section{Problem Statement}\label{sec:problem_statement}
Users (or agents) arrive sequentially over a period $T$ on a platform $\pfm$ and make choices.  The context vector corresponding to an agent arriving at time step $t \in [T]$ is represented as $\theta_t \in \R^d$ (w.l.o.g. assume $\lVert\theta_t\rVert_2 \leq 1$). Each choice is represented as an arm $i \in \arms$ (with $|\arms|=\narms$), which is associated with a fixed $d$-dimensional attribute vector $\mu_i$ (w.l.o.g. assume $\lVert\mu_i\rVert_2 \leq 1$). We can think of each coordinate of $\mu_i$ as an attribute of arm $i$ that may influence the user to choose it over the others. True arms attributes are unknown to both platform and the agents a priori, and the platform shows its estimate of these attributes to arriving agents.\\

\noindent{\textbf{User choice and reward model:}} The user choice behavior is myopic in nature: she is presented with the empirical estimates of $\{\mu_i\}_{i \in \arms}$: $\{\hat{\mu}_i\}_{i \in \arms}$, corresponding to the arms available on the platform (e.g., via metadata, tags or auxiliary textual information) and then she makes a singleton choice. In this notation, $\est^t_i$ denotes the latest estimate for the arm $i$ available at the time $t$. She may have a random utility for each arm $i$, whose mean is $\theta_t.\mu_i$ (an inner product), where $\theta_t$ is her context vector. Given these utilities, she picks an arm with the highest perceived utility. In the special case where there is no randomness in the utilities, then her decision is simply $\argmax_{j \in \arms} \theta_t.\hat{\mu}_j$. For simplicity, we will work under this restriction for the rest of the paper. Let the chosen arm be denoted as $i_t$ at round $t$. The reward accrued by the user is $\theta_t.\mu_{i_t}$.\\

\noindent{\textbf{Feedback model:}} Although the platform keeps track of all interaction history, it can only observe the context after the agent has arrived on the platform. The platform computes and displays the empirical estimates $\{\hat{\mu}_i\}_{i \in \arms}$ based on the measurements it is able to make. The measurements include the context of the user that arrived and the random utility that she obtained: $y_t = \theta_t.\mu_{i_t} + \eta_t$, where $\eta_t$ is a zero mean i.i.d. sub-Gaussian noise random variable. The platform estimates $\{\hat{\mu}_i\}_{i \in \arms}$ by using the observed contexts and the reward signals for each arm at each time step, most often by solving a regression problem. Some useful notations are as follows: $\Theta$ is the $T\times d$-dimensional \textit{design matrix} whose rows are the contexts $\theta_t$. Also $\forall i \in \arms$, $ S_{i,t} := \{ s < t \vert i_s = i \}$. Further, $\Theta(S_{i,t})$ represents the design matrix corresponding to the contexts arriving at the time steps denoted by $S_{i,t}$, and $Y(S_{i,t})$ denotes the collection of rewards corresponding to these contexts at time steps $S_{i,t}$.\\ 

\noindent{\textbf{Learning objective:}} The platform incurs an instantaneous regret $r_t$ if the arm picked by the user is not the best arm for that user. That is, $r_t = \max_{j}\theta_t.\mu_j - \theta_t.\mu_{i_t}$. The goal of the platform is to reduce the expected cumulative regret $\regret = \mathbb{E}[\sum_{t=1}^{T}r_t]$ over the horizon $T$. Intuitively, if the platform had the knowledge and could display the true attributes of the arms, then the users would pick the items that are best suited to them, and the cumulative regret would be zero. But since the platform does not know the attributes of the arms a priori and the users are acting myopically, it has to incentivise some of these users to explore (based on the history of contexts and rewards generated thus far). The platform does so by displaying a payment/discount vector $\mathbf{p}^t$ in addition to the estimated arm attributes. The corresponding user's decision is $\argmax_{j \in \arms} (\theta_t.\hat{\mu}_j+p^t_{j})$. The goal of the platform is to design incentivization schemes that minimize the cumulative regret, while keeping the total payments made as small as possible. We assume all ties to be broken arbitrarily. Hence at each round $t$, an agent with context $\theta_t$ (unknown to the platform when it is deciding payments) arrives on the platform. The platform presents the agent with arm estimates $\{\hat{\mu}_i\}_{i \in \arms}$ and a payment vector $\mathbf{p}^t$. The agent makes a singleton choice, thereby accruing some reward. The platform observes the context and a noisy measurement of this reward, and updates its estimates.

\section{Algorithms and Guarantees}\label{sec:algorithm}
In this section, we propose a new algorithm (\PayToAddNoise, see Algorithm \ref{alg:Pay_all_Rounds}) that uses randomized payments to incentivize agents, enabling the platform to incur sub-linear regret. Essentially we identify a way to adapt and extend the non principal-agent setting of ~\cite{kannan18} to our platform-user interaction model. One way to reduce the cost that the platform incurs towards incentivization is to work with a special class of contexts (those having \emph{covariate diversity}, see Definition~\ref{def:covd}), which would provide \emph{exploration} of the arms naturally, leading to learning and low-regret. More specifically, in the contextual bandit setting of~\cite{kannan18}, the authors assume that a known perturbation (i.i.d. noise) is added to the contexts before they are picked up by the platform. They show that because of this perturbation the power of adversary (in choosing the contexts) is reduced and a myopic selection of arms enjoys sublinear regret (Theorem~\ref{thm:regret}).

In our setting, the choice of context at a given round is purely adversarial and we make no assumption on the contexts. Our key idea is to \emph{use payments to mimic perturbations}. We show that with the proposed payment scheme, \emph{covariate diversity} can be infused into our model, even if the arriving contexts are adversarial. Finally, we bound the expected cumulative payments in our scheme and show that it is sub-linear in $T$.

Our algorithm \PayToAddNoise{} is described in Alg.~\ref{alg:Pay_all_Rounds}. The key idea is to first generate perturbations that can satisfy the covariate diversity condition, and then transform these perturbations to a payment vector, which is then presented to the user. The user then myopically picks the best action, given these payments (one for each arm), ensuring fair compensation if this choice was different from their original choice. The platform updates the estimates of the selected arm's attribute vector by performing a regression while taking the payment information into account. As we show below, this approach enjoys sublinear (in horizon $T$) upper bounds on regret and the payment budget.

\begin{algorithm}
\SetKwFunction{CalcPayment}{CalcPayment}
\SetKwFunction{UpdateEstimate}{UpdateEstimate}
\SetKwFunction{InitialExploration}{InitialExploration}
\textbf{Input:} Arms: $\arms$, time horizon: $T$, and initial exploration parameter: $m$.\\
\InitialExploration{}\\
\For{ $t=m+1$ to $T$}{
Agent with context $\theta_t$ arrive at the Platform.\\
$\{p^t_{i}\}_{i \in \arms}$ = \CalcPayment{}.\\
Agent choose arm $\pi_t = \arg\max_i (\est_i^t.\theta_t \, + p^t_{i})$.\\
\UpdateEstimate{}\\
}
\SetKwProg{myproc}{Procedure}{}{}
\myproc{\CalcPayment{}}{
$p^t_i= \zeta_t. \est^t_i$, where $\zeta_t \sim \mathcal{N}(0,\sigma^2 I_d)$ for all arms.\\
}
\myproc{\UpdateEstimate{}}{
\textit{Updating History:}\\
\quad\quad $\Theta(S_{\pi_t,t+1}) = [\Theta(S_{\pi_t,t}) \vert (\theta_t+\zeta_t)] \; $ with $\zeta_t$ obtained above, and\\ \quad\quad $Y(S_{\pi_t,t+1}) = [Y(S_{\pi_t,t})\vert (\est_{\pi_t}.\theta_t+p^t_{\pi_t})]$.\\
\textit{Updating Parameter:}\\
\quad\quad $\est_{\pi_t}^{t+1} = (\Theta(S_{\pi_t,t})^T\Theta(S_{\pi_t,t}))^{-1}\Theta(S_{\pi_t,t})^TY(S_{\pi_t,t})$.\\
}
\caption{\PayToAddNoise}\label{alg:Pay_all_Rounds}
\end{algorithm}

\begin{lemma}\label{lem:paymentEquivalence}
In \PayToAddNoise{} (Algorithm ~\ref{alg:Pay_all_Rounds}), there exists a suitable payment for each arm such that $\arg\max_i (\est_{i}^t.\theta_t + p_i^t) = \arg\max_i \est_{i}^t.(\theta^{\circ}_t)$ for all $t>m$ ($m$ is the number of initial forced exploration rounds). And $\theta^{\circ}_t$ satisfies covariate diversity (Definition~\ref{def:covd}). Additionally, expected payments made by the platform are sub-linear in horizon $T$, specifically the average cumulative payments are $\mathcal{O}\left(N\sqrt{2T\log(NT)}\right)$.
\end{lemma}
\begin{proof}
First, we make some observations. The platform can offer \textit{negative} payments implying users would incur some penalty if they select certain actions. Hence, the platform can influence the choice of the myopic user by providing a collection of payments and penalties (one for each arm). Enforcing payments as: $p^t_i = \zeta_t.\est^t_i$ where $\zeta_t \sim \mathcal{N}(0,\sigma^2I_d)$, ensures that the perceived context, $\theta_t+\zeta_t$ at any given round $t$ satisfies the \emph{covariate diversity condition}. Hence, in the proposed payments scheme, the platform pays a random payments vector $\mathbf{p}^t$ where each arm may receive a non-zero value, depending on the estimates $\est^t$.

The cumulative payments for an arm $i$ can be expressed as:
\begin{equation}\label{eq:cum_pay}
    Payment(T,i) = \sum_{t=1}^T\zeta_t.\est^t_{i_t},
\end{equation}
Notice that, $\zeta_t.\est^t_i$ is a sum of sub-Gaussian random variables as  $\zeta_t.\est^t_i = \sum_{l=1}^d \zeta^{(l)}_t.\est^{(l),t}_i$. Hence $\zeta_t.\est^t_i$ is a sub-Gaussian random variable with the variance-proxy parameter, $||\est^t_i||$. Since we assume that  $||\mu_i||\leq 1$, estimate (in our algorithm) $||\est^t_i||\leq 1$ as well. Thus we can use sub-Gaussian tail bounds to upper bound the absolute value of the payments in Eq. (\ref{eq:cum_pay}).
Consider the following standard tail bound for sub-Gaussian random variable:
\begin{lemma}
Let $Y_1,Y_2..Y_t$ be an s-sub-Gaussian martingale, i.e, each $Y_j$ is distributed as mean-0 and s-sub-Gaussian conditioned on $Y_1,..Y_{j-1}$. Then:
\begin{equation*}
    \mathbb{P}\left[\sum_{j=1}^t Y_j < \sqrt{2ts\log(1/\delta)}\right] > 1-\delta
\end{equation*}
\end{lemma}
Thus we bound the sum $\sum_{t=1}^T\zeta_t.\est^t_{i_t}$ with probability at least $1-\delta$ with the quantity:
\begin{equation}\label{eq:paymentbound}
    \sum_{t=1}^T\zeta_t.\est^t_{i_t} < \sqrt{2T\log(1/\delta)}.
\end{equation}
In Eq (\ref{eq:paymentbound}), we apply a union bound to obtain a bound for all arms $i\,\in\, \arms$ simultaneously with probability $1-\delta'$, as shown below:
\begin{equation*}
    \sum_{t=1}^T\zeta_t.\est^t_{i_t} < \sqrt{2T\log(N/\delta')}
\end{equation*}
Hence, the cumulative payments across all arms is upper bounded by:
\begin{equation*}
    \sum_{i=1}^N\,Payments(T,i) < N\sqrt{2T\log(N/\delta')}, 
\end{equation*}
with probability at least $1-\delta'$.
To realize the final bound we use $\delta'=1/T$.
\end{proof}

We now provide a  proof of the regret claim. First, we re-write the definition of \emph{covariate diversity} from~\cite{kannan18} as below.
\begin{definition}\label{def:covd}
For any distribution $\mathcal{D}$ with $\zeta\,\sim\,\mathcal{D}$ and $\zeta\,\in\,\mathbb{R}^d$ and $\theta^{\circ}_t := \theta_t+\zeta$, for any arbitrary $\theta_t\,\in\,\mathbb{R}^d$ such that: (a) if $\zeta$ is a "centrally bounded", i.e. $w.\zeta\leq r\,,\forall w:||w||\leq 1$ with high probability, and (b) if the minimum eigenvalue of the expected outer product $\mathbb{E}[\theta^{\circ}_t.(\theta^{\circ}_t)^T]$ is lower bounded, i.e: 
\begin{equation*}
    \lambda_{\min}\left[\mathbb{E}\left[\theta^{\circ}_t.(\theta^{\circ}_t)^T\right]\right] \geq \lambda_{\circ},
\end{equation*}
then, the perturbed context, $\theta^{\circ}_t$ has \emph{covariate diversity}.
\end{definition}
\begin{remark}
In the Algorithm~\ref{alg:Pay_all_Rounds}, an agent makes a choice after receiving the payment vector from the platform and hence to the platform, the perceived context $\theta^{\circ}_t$ has Gaussian ("centrally bounded" distribution) perturbation baked-in providing co-variate diversity to the context. Such a condition on the context implies that there is non-trivial variance in all dimensions and intuitively such an arrangement allows convergence of the least square estimator of arm attributes.
\end{remark}

Since (a) the payments scheme proposed in the proof of Lemma~\ref{lem:paymentEquivalence} establishes \emph{covariate diversity}, and (b) in the Algorithm \ref{alg:Pay_all_Rounds}, we update history with perturbed contexts, it is intuitive to see that the regret upper bound of Theorem 4.1 of \cite{kannan18} (derived in the non principal-agent setting) also applies here.
\begin{theorem}\label{thm:regret}
With an appropriate initial exploration (parameterized by $m$), \PayToAddNoise{} has the following regret upper bound with probability at least $1-\delta''$:
\begin{equation*}
    R(T) \leq \mathrm{\tilde{O}}\left(\sqrt{TN}\log\left(TN\right)^{3/2}\right),
\end{equation*}
where the notation $\mathrm{\tilde{O}}(.)$ hides dependence on instance specific parameters and $\delta''$.
\end{theorem}

\begin{remark}
Note that for the regret guarantee to hold, Algorithm \ref{alg:Pay_all_Rounds} must have an initial exploration phase, during which the agents are made to play arms uniformly at random or in a round-robin fashion. Intuitively, this warm-start is required to build up robustness of estimates against adversarial contexts. 
\end{remark}

\subsection{Other Payments scheme and Lower Bound}
In the previous section, we established a payments scheme with bounded cumulative cost to the platform that also allowed for sub-linear regret without any additional assumption on the the instance or the adversarial choice of the contexts. It is natural to ask the following question: does there exist a payments scheme which is even more frugal for the platform (i.e., costs less) and still ensures sub-linear regret? Could there be a principal-agent setting where initial exploration is not needed? The first question has been partially addressed before. In~\cite{chen2018incentivizing}, the authors show that only a constant (in $T$) total amount of payment is required for a sub-linear regret bound. However, in their model the platform knows the distribution of the contexts as well as views the context of the arriving agent before deciding on the payments, this is in addition to the heterogeneity assumption on the contexts, which is equivalent to the \emph{covariate diversity} described above. In~\cite{kannan2017fairness}, the authors presents a randomized algorithm which does not need any initial exploration phase as the exploration is baked-into the randomization. Their scheme, however requires that the agents and the platform maintain the estimate of the arm attributes using a ridge estimator. 

In the previous section and in the above works, cumulative payment scales up with instance parameters. We claim that, this is essential if we ought to perform better than a vanilla explore-then-commit strategy\footnote{In a typical explore-then-commit learning strategy, there is an initial pure exploration phase by the end of which the learner commits to a single best action till the end of the horizon $T$~\cite{langford2007epoch}}, as shown in the following lemma.
\begin{lemma}
    Consider $\mathcal{A}$ to be the set of all explore-then-commit algorithms (without incentivization) for the contextual bandit that does not make any addition assumptions on the instance or the contexts. With a restricted upper cap $B$ on the cumulative payments budget, no algorithm can do better than the best algorithm in the set $\mathcal{A}$ even with an initial exploration.
\end{lemma}
\begin{proof}
Firstly, we make an observation that the best algorithm (denoted by \textit{Alg}) in the set $\mathcal{A}$: it has the best regret guarantee of all  algorithms that do not explicitly incentivize by payments and have an initial exploration phase. Consider an instance with two arms and let $t$ be the first round after the initial exploration phase. Let $\hat{\mu}_1$ and $\hat{\mu}_2$ be the corresponding estimates of the arm attributes, visible to the arriving agents on the platform. As the agent arrival is purely adversarial,  $\exists$ context $\theta'$, such that $(\hat{\mu}_1-\hat{\mu}_2)\cdot \theta' > B$. Further, if the adversary opts for this context for all the following rounds till $T$, then incentivizing through payments is fruitless. This is because, the fixed budget $B$ is too less to induce any change to the myopic behavior of the agents. Hence, in fixed budget regimes, \textit{Alg} has the best regret guarantee.
\end{proof}

\section{Simulations}\label{sec:simulation}

In this section we compare the learning performance (regret) and payment requirements for our proposed strategy Algorithm \ref{alg:Pay_all_Rounds} and other standard baselines for both synthetic and real datasets. For ease of referencing we name the algorithms as: (1) \PayToAddNoise~(Algorithm \ref{alg:Pay_all_Rounds}); (2) \PaymentProfile~(an algorithm in which the platform provides as much payment as required so that the myopic agents choose arms as if they are deploying LinUCB~\cite{li2010contextual}); (3) \UnrestrictedBudget~(an algorithm based on the chaining method of~\cite{kannan2017fairness}); (4) \LimitedBudget~(an instance of the algorithm \UnrestrictedBudget with a fixed upper cap on the total cumulative payments) and (5) \textsc{NoPayments}~(the platform is passive and agents make myopic choice without any influence).

In our first experiment, the contexts are drawn from a multivariate Gaussian distribution with a non-zero mean. We set the number of arms to be $|\arms|=8$, the context dimension as $d=4$, and the time horizon as $T=800$, while averaging over $10$ Monte Carlo runs (refer to Fig~\ref{fig:1}). The \textsc{NoPayments} strategy, i.e., where the platform has no control on exploration, perform very well and has a sub-linear regret. However, in our simulation studies its performance was consistently surpassed by other algorithms, especially \PaymentProfile~with LinUCB as the underlying strategy. One interesting result (which is also observed in the next experiment) is that \PayToAddNoise~has good performance in terms of payments required to ensure sub-linear regret. This reinforces our theoretical guarantees for the same (see Lemma~\ref{lem:paymentEquivalence}, where upper bounds on the expected total payments were stated). On the other hand, LinUCB (implemented within~\PaymentProfile) incurred large incentivization costs in these synthetic principal-agent instances.

\begin{figure}
\centering
\includegraphics[width=.45\columnwidth]{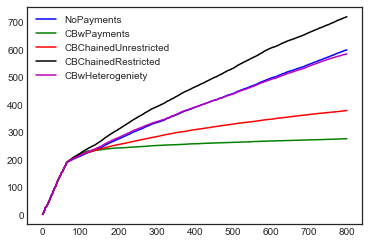}
\includegraphics[width=.45\columnwidth]{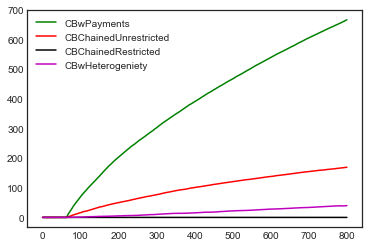}
\caption{Left plot shows cumulative regret, right shows the total payments made by various algorithms. In both plots, x-axis is the time horizon and y-axis represents either cumulative regret or cumulative  payments made.}
\label{fig:1}
\end{figure}

Next, we use the same experimental setup as before, but use a publicly available data set to mimic arm attribute learning: the \href{https://www.openml.org/d/1471}{EEG} data set from the \href{https://www.openml.org}{OpenML} platform. This data set contains $14$-dimensional feature vectors with two possible class labels ($|\arms| = 2$). We use this classification instance to generate contexts and assign rewards. We standardize the feature vectors as a pre-processing step. Taking the time horizon as $T=2500$, we randomize the arrival of contexts and  report results averaged over $10$ Monte Carlo runs (refer Fig~\ref{fig:1}). Interestingly, the \textsc{NoPayments} strategy performs very well, followed by the payment based schemes (note that our algorithm is quite competitive in this setting and has regret and payment guarantees while \textsc{NoPayments} does not without addtional assumptions). 

\begin{figure}
\centering
\includegraphics[width=.45\columnwidth]{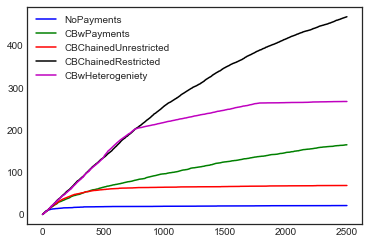}
\includegraphics[width=.45\columnwidth]{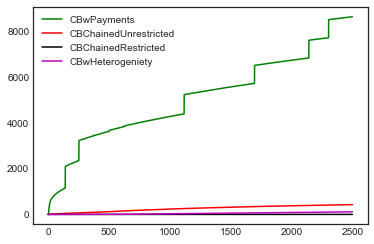}
\caption{Left plot shows the cumulative regret, right shows the total payments made by various algorithms. In both, x-axis is the time horizon and y-axis represents either cumulative regret or cumulative total payments.}
\label{fig:2}
\end{figure}

\section{Conclusion}
In this paper, we studied the principal-agent variants of online learning under the contextual bandit framework, where a platform sends recommendations and users act on those that are most valuable to them, and the platform can use payments to incentivize exploration and fasten learning. 

This paper is among only a handful of recent works which have tackled the the problem of incentivization/recommendation in principal-agent settings, hence several fruitful avenues for extending this initial foray remain.
\begin{itemize}
    \item In Algorithm \ref{alg:Pay_all_Rounds}, platform uses payments to infuse heterogeneity in the arriving contexts. It is easy to ensure sub-linear regret with $\Omega(T)$ payments. Similarly, if the allowed regret is upto $\mathrm{O}(T)$, the platform does not need to pay at all. It would an interesting problem to calculate lower bounds on payments required for a reasonable regret guarantee.
    \item It seems to be the case that notions such as covariate diversity may be necessary for unbiased estimation of arm attributes. Hence, a study which ties together the efficacy of various algorithms (including ours) to covariate diversity in the contexts could be an interesting contribution in the incentivized exploration literature.
    \item Although assuming myopic behavior of the agents is an intuitive modeling choice, it may not cover all the practical possibilities. Hence, extending algorithm design and analysis to situations where the agents are non-myopic, for instance, they are anticipating payments, are partially observed, or are governed by a rich discrete choice model. All these would also be of significant interest.
    \item More complex user behaviors can modeled if the platform can inform the estimate of each arm's attributes along with their variance. This can better inform the users, especially the ones that are risk-averse.
\end{itemize}

\bibliographystyle{plain}
\bibliography{pay2explore}

\end{document}